\documentclass{article}

\PassOptionsToPackage{numbers,sort&compress}{natbib}

\usepackage[preprint, nonatbib, final]{iaseai26}

\usepackage[utf8]{inputenc} %
\usepackage[T1]{fontenc}    %
\usepackage[colorlinks=true, allcolors=blue]{hyperref}
\usepackage{url}            %
\usepackage{booktabs}       %
\usepackage{amsfonts}       %
\usepackage{nicefrac}       %
\usepackage{microtype}      %
\usepackage{xcolor}         %

\usepackage[
  backend=biber,
  style=numeric-comp,   %
  sortcites=true,
  sorting=none,         %
  natbib=true,          %
  giveninits=true,
  maxcitenames=4,       %
  minbibnames=3,        %
  maxbibnames=4,        %
  doi=false,            %
  url=false,
  hyperref=true
]{biblatex}
\DeclareSourcemap{
  \maps[datatype=bibtex]{
    \map{
      \step[fieldset=file,      null]
      \step[fieldset=abstract,  null]
      \step[fieldset=keywords,  null]
      \step[fieldset=urldate,   null]
      \step[fieldset=month,     null]
      \step[fieldset=day,       null]
      \step[fieldset=language,  null]
      \step[fieldset=isbn,      null]
      \step[fieldset=issn,      null]
      \step[fieldsource=doi,
            match=\regexp{^\s*https?://(dx\.)?doi\.org/},
            replace={}]
      \step[fieldsource=booktitle,
            match=\regexp{^\s*Proceedings\s+of\s+(the\s+)?},
            replace={}]
      \step[fieldsource=eventtitle,
            match=\regexp{^\s*Proceedings\s+of\s+(the\s+)?},
            replace={}]
      \step[fieldsource=booktitle,
            match=\regexp{^\s*(19|20)\d{2}\s+},
            replace={}]
      \step[fieldsource=eventtitle,
            match=\regexp{^\s*(19|20)\d{2}\s+},
            replace={}]
      \step[fieldsource=booktitle,
            match=\regexp{^\s*\d+(st|nd|rd|th)\s+},
            replace={}]
      \step[fieldsource=eventtitle,
            match=\regexp{^\s*\d+(st|nd|rd|th)\s+},
            replace={}]
      \step[fieldsource=note,
            match=\regexp{.*arXiv:([0-9]{4}\.[0-9]{4,5})(v[0-9]+)?.*},
            replace={$1},
            fieldtarget=eprint]
      \step[fieldsource=eprint,
            match=\regexp{^[0-9]{4}\.[0-9]{4,5}(v[0-9]+)?$},
            fieldset=eprinttype,
            fieldvalue={arxiv}]
      \step[fieldsource=eprint,
            match=\regexp{^([0-9]{4}\.[0-9]{4,5}(v[0-9]+)?)$},
            replace={https://arxiv.org/abs/$1},
            fieldtarget=url]
    }
  }
}
\renewbibmacro{in:}{}

\newbibmacro*{title:doiurl}[1]{%
  \iffieldundef{doi}
    {\iffieldundef{url}
       {#1}
       {\href{\thefield{url}}{#1}}}
    {\href{https://doi.org/\thefield{doi}}{#1}}%
}

\DeclareFieldFormat{title}{%
  \usebibmacro{title:doiurl}{\mkbibemph{#1}}%
}
\DeclareFieldFormat[article,incollection,techreport,inproceedings,book,misc]{title}{%
  \mkbibquote{\usebibmacro{title:doiurl}{#1}}%
}

\usepackage[dvipsnames]{xcolor}
\definecolor{porange}{HTML}{E77500} %
\usepackage{multicol}
\usepackage{amsfonts}
\usepackage{amsthm}
\usepackage{fancyhdr}
\usepackage{comment}
\usepackage{amsmath}
\usepackage{changepage}
\usepackage{amssymb}
\usepackage{graphicx}
\usepackage{adjustbox}
\usepackage{latexsym}
\usepackage{enumerate}
\usepackage{cleveref}
\usepackage{float}
\usepackage{subcaption}

\usepackage{mathtools}
\usepackage[ruled,vlined,linesnumbered]{algorithm2e}
\usepackage{bbm}
\usepackage{bm}
\usepackage{booktabs}
\usepackage[xindy, acronym, nowarn]{glossaries}   %
\usepackage{multirow}
\usepackage{makecell}
\usepackage{colortbl}
\usepackage{balance}

\usepackage{textcomp}
\usepackage{stfloats}  %

\usepackage{fontawesome5}  %

\usepackage{algorithmic}

\usepackage{amsthm}
\newtheorem{theorem}{Theorem}
\newtheorem{lemma}{Lemma}    

\newtheorem{remark}{Remark}

\newtheorem{proposition}{Proposition}    
\newtheorem{assumption}{Assumption}    
\newtheorem{definition}{Definition}

\newtheorem*{proposition*}{Proposition}
\newtheorem*{theorem*}{Theorem}

\definecolor{porange}{HTML}{E77500} %

\DeclareDocumentEnvironment{example}{}{\noindent\textbf{Running example:}\itshape}{}

\usepackage{ifthen}
\newboolean{include-notes}
\newboolean{highlight-new}
\newboolean{include-remove}
\setboolean{include-notes}{true}
\setboolean{highlight-new}{true}
\setboolean{include-remove}{true}

\usepackage[dvipsnames]{xcolor}
\usepackage[normalem]{ulem}
\newcommand{\jaime}[1]{\ifthenelse{\boolean{include-notes}}{\textcolor{orange}{\textbf{Jaime:} #1}}{}}
\newcommand{\haimin}[1]{\ifthenelse{\boolean{include-notes}}{\textcolor{magenta}{\textbf{Haimin:} #1}}{}}
\newcommand{\justin}[1]{\ifthenelse{\boolean{include-notes}}{\textcolor{Cerulean}{\textbf{Justin:} #1}}{}}
\newcommand{\himani}[1]{\ifthenelse{\boolean{include-notes}}{\textcolor{Plum}{\textbf{Nishanth:} #1}}{}}
\newcommand{\david}[1]{\ifthenelse{\boolean{include-notes}}{\textcolor{teal}{\textbf{David:} #1}}{}}
\newcommand{\remove}[1]{\ifthenelse{\boolean{include-remove}}{\textcolor{red}{\sout{#1}}}{}}
\newcommand{\new}[1]{\ifthenelse{\boolean{highlight-new}}{\textcolor{blue}{#1}}{#1}}
\newcommand{\todo}[1]{\ifthenelse{\boolean{include-notes}}{\textcolor{blue}{\textbf{TODO:} #1}}{}}
\newcommand{\guy}[1]{\ifthenelse{\boolean{include-notes}}{\textcolor{Cerulean}{\textbf{Guy:} #1}}{}}
\newcommand{\flag}[1]{\ifthenelse{\boolean{include-notes}}{\textcolor{red}{#1}}{#1}}

\newcommand{\princeton}[1]{\ifthenelse{\boolean{include-notes}}{\textcolor{orange}{#1}}{}}

\newcommand{\p}[1]{\smallskip \noindent \textbf{{#1}.}}

\newcommand{\eg}{\textit{e.g.}}
\newcommand{\ie}{\textit{i.e.}}

\newcommand{\failureset}{{\mathcal{F}}}
\newcommand{\failure}{{\failureset}}

\newcommand{\support}{\operatorname{supp}}

\newcommand{\st}{\textnormal{s.t.}}

\newcommand{\state}{{s}}

\newglossarystyle{mylist}{%
  \setglossarystyle{list}%
}

\glsdisablehyper
\makeglossaries

\newglossaryentry{RL}
{
  name={RL},
  description={reinforcement learning},
  first={reinforcement learning (\glsentrytext{RL})}
}

\newglossaryentry{HJ}
{
  name={HJ},
  description={Hamilton--Jacobi},
  first={Hamilton--Jacobi (\glsentrytext{HJ})}
}

\newglossaryentry{DCBF}
{
  name={DCBF},
  description={discrete-time control barrier function},
  first={discrete-time control barrier function (\glsentrytext{DCBF})}
}

\newglossaryentry{CBF}
{
  name={CBF},
  description={control barrier function},
  first={control barrier function (\glsentrytext{CBF})}
}

\newglossaryentry{Q-CBF}
{
  name={Q-CBF},
  description={state--action control barrier function},
  first={state--action control barrier function (\glsentrytext{Q-CBF})}
}

\newglossaryentry{ODD}
{
  name={ODD},
  description={Operational Design Domain},
  first={operational design domain (\glsentrytext{ODD})}
}

\newglossaryentry{LRSF}
{
  name={LRSF},
  description={Last-Resort Safety Filter},
  first={last-resort safety filter (\glsentrytext{LRSF})}
}

\newglossaryentry{HCSF}
{
  name={HCSF},
  description={Human-Centered Safety Filter},
  first={human-centered safety filter (\glsentrytext{HCSF})}
}

\newglossaryentry{NLP}
{
  name={NLP},
  description={nonlinear programming problem},
  first={nonlinear programming problem (\glsentrytext{NLP})},
}

\newglossaryentry{ILQR}
{
  name={ILQR},
  description={iterative linear quadratic regulator},
  first={iterative linear quadratic regulator (\glsentrytext{ILQR})},
}

\newglossaryentry{MPC}
{
  name={MPC},
  description={model predictive control},
  first={model predictive control (\glsentrytext{MPC})},
}

\newglossaryentry{AC}
{
  name={AC},
  description={Assetto Corsa},
  first={Assetto Corsa (\glsentrytext{AC})},
}

\newglossaryentry{SAC}
{
  name={SAC},
  description={Soft Actor--Critic},
  first={soft actor--critic (\glsentrytext{SAC})},
}

\newglossaryentry{ANOVA}
{
    name={ANOVA},
    description={Analysis of Variance},
    first={analysis of variance (\glsentrytext{ANOVA})}
}

\newglossaryentry{SME}
{
    name={SME},
    description={Simple Main Effects},
    first={simple main effects (\glsentrytext{SME})},
}

\newglossaryentry{HSD}
{
    name={HSD},
    description={Honestly Significant Difference},
    first={honestly significant difference (\glsentrytext{HSD})},
}

\newglossaryentry{ECDF}
{
    name={ECDF},
    description={Empirical Cumulative Distribution Function},
    first={empirical cumulative distribution function (\glsentrytext{ECDF})},
}

\newglossaryentry{OCP}
{
    name={optimal control problem},
    description={Optimal Control Problem},
    first={optimal control problem}
}

\newglossaryentry{AI}
{
    name={AI},
    description={Artificial Intelligence},
    first={artificial intelligence (\glsentrytext{AI})}
}

\newglossaryentry{HRI}
{
    name={HRI},
    description={Human--Robot Interaction},
    first={human--robot interaction (\glsentrytext{HRI})}
}

\newglossaryentry{MDP}
{
    name={MDP},
    description={Markov Decision Process},
    first={Markov decision process (\glsentrytext{MDP})}
}

\newglossaryentry{SC-MDP}
{
    name={SC-MDP},
    description={Safety-Critical Markov Decision Process},
    first={safety-critical Markov decision process (\glsentrytext{SC-MDP})}
}

\newacronym[longplural={constrained Markov decision processes}]{CMDP}{CMDP}{constrained Markov decision process}

\newglossaryentry{DP}
{
    name={DP},
    description={Dynamic Programming},
    first={dynamic programming (\glsentrytext{DP})}
}

\newglossaryentry{CPO}
{
    name={CPO},
    description={Constrained Policy Optimization},
    first={Constrained Policy Optimization (\glsentrytext{CPO})}
}

\newglossaryentry{RCPO}
{
    name={RCPO},
    description={Reward Constrained Policy Optimization},
    first={Reward Constrained Policy Optimization (\glsentrytext{RCPO})}
}

\newacronym{MPSF}{MPSF}{model predictive safety filter}

\newacronym{ISAACS}{ISAACS}{Iterative Soft Adversarial Actor–Critic for Safety}

\newacronym{QP}{QP}{quadratic program}

\newacronym[longplural={Gaussian processes}]{GP}{GP}{Gaussian process}

\newacronym{CVaR}{CVaR}{conditional value-at-risk}

\usepackage{cleveref}

\title{Provably Optimal Reinforcement Learning\\under Safety Filtering}

\author{
  Donggeon David Oh%
  \thanks{D. D. Oh and D. P. Nguyen contributed equally.}%
  \hspace{0.4em}\thanks{Corresponding author: \texttt{do9948@princeton.edu}}%
  \hspace{0.4em}\thanks{Department of Electrical and Computer Engineering, Princeton University}%
  \And
  Duy P.~Nguyen%
  \footnotemark[1]\hspace{0.4em}\footnotemark[3]
  \And
  Haimin Hu%
  \footnotemark[3]\hspace{0.4em}\thanks{Department of Computer Science, Johns Hopkins University}
  \And
  Jaime Fernández~Fisac%
  \footnotemark[3]
}

\begin{document}

\maketitle
\setcounter{footnote}{0}
\begin{abstract}

Recent advances in \gls{RL} enable its use on increasingly complex tasks, but the lack of formal safety guarantees still limits its application in safety-critical settings.
A common practical approach is to augment the RL policy with a \emph{safety filter} that overrides unsafe actions to prevent failures during both training and deployment. However, safety filtering is often perceived as sacrificing performance and hindering the learning process.
We show that this perceived safety--performance tradeoff is not inherent and prove, for the first time, that enforcing safety with a sufficiently permissive safety filter does \emph{not} degrade asymptotic performance. We formalize RL safety with a \emph{safety-critical Markov decision process (SC-MDP)}, which requires categorical, rather than high-probability, avoidance of catastrophic failure states. Additionally, we define an associated \emph{filtered MDP} in which all actions result in safe effects, thanks to a safety filter that is considered to be a part of the environment.
Our main theorem establishes that (i) learning in the filtered MDP is safe categorically, (ii) standard RL convergence carries over to the filtered MDP, and (iii) any policy that is optimal in the filtered MDP—when executed through the same filter—achieves the \emph{same} asymptotic return as the best safe policy in the SC-MDP, yielding a complete separation between safety enforcement and performance optimization. We validate the theory on Safety Gymnasium with representative tasks and constraints, observing zero violations during training and final performance matching or exceeding unfiltered baselines. Together, these results shed light on a long-standing question in safety-filtered learning and provide a simple, principled recipe for safe RL: train and deploy RL policies with the most permissive safety filter that is available.
\end{abstract}

\section{Introduction}
\label{sec:intro}
\Glsfirst{RL} has demonstrated outstanding performance in complex domains, yet a fundamental limitation is the lack of
safety guarantees, both during policy training and at deployment.
This poses a major barrier to deploying \gls{RL} in safety-critical applications, from autonomous driving to healthcare management,
where even a single safety violation would be catastrophic.
One promising approach towards safety assurances in RL is to integrate a \textit{safety filter}~\cite{hsu2024safety} that monitors the system and overrides any unsafe candidate actions to preempt downstream failures.
In effect, the \gls{RL} agent can explore and learn \textit{within} a safe set, while the filter ensures that it never enters unsafe regions.

However, enforcing hard safety constraints during training is commonly understood to negatively interfere with the learning process and limit the asymptotic performance achievable by the learned policy~\cite{gros2020safe, anderson2020neurosymbolic, yang2023safe, koller2019learning}.
In addition, prior studies on safety filtering in control systems~\cite{leung2020infusing,hu2024active,bejarano2024safety} have reported that if the task policy is unaware of the presence of a safety filter, it may repeatedly attempt unsafe actions that keep triggering filter overrides, resulting in suboptimal oscillations or ``chattering'' behaviors.
In such cases, the separation between safety and performance is incomplete: the task policy must be made explicitly aware of the task-agnostic safety filter in order to avoid costly overrides.
Analogous concerns appear in the broader AI safety literature, where safety alignment
of large reasoning models has been reported to degrade reasoning capability under specific sequential
fine-tuning pipelines~\cite{huang2025safety}.
 
These concerns raise a fundamental, unsettled question in safe \gls{RL}: \textit{Does safety filtering inevitably hinder the agent’s learning and ultimately limit the achievable performance?}
We provide, for the first time, a theoretical answer to this question by proving that no such entanglement between safety filtering and task policy training is necessary to obtain an optimal constrained policy.
We view this as a new foundational result in safe \gls{RL}, with significant practical implications as AI technology advances towards increasingly complex autonomous systems.

\textbf{Contributions.} 
\textit{Our key insight is that safety and performance in \gls{RL} can be learned separately and still converge towards provably safe and task-optimal agent behavior.}
We prove that, under a safety filter that is least-restrictive yet capable of preventing all safety violations, an \gls{RL} agent trained in the filtered environment converges to the \emph{optimal} safety-constrained policy; in particular, the convergence guarantees enjoyed by \gls{RL} algorithms on stationary, discounted \glspl{MDP} with bounded rewards \emph{carry over} to this setting~\cite{watkins1992q, tsitsiklis1994asynchronous, konda1999actor, haarnoja2018soft}.
Remarkably, we show—\emph{for the first time}—that the \emph{safety–performance separation} is complete in safe \gls{RL}: the task policy can be trained entirely agnostic to the safety filter and still attain the \emph{same} asymptotic performance as if it had been trained with the safety constraints in mind.
The safety filter guarantees that the agent never leaves the safe set, wherein the \gls{RL} algorithm, free of the extra burden of handling safety, optimizes the task performance.
To complement our theoretical findings, we empirically validate the safe \gls{RL} principle using the Safety Gymnasium benchmark~\cite{ji2023safety}, a modern successor to OpenAI’s Safety Gym~\cite{OpenAI2019SafetyGym}, across different tasks, constraints, and environments.

\begin{figure}[t]
    \centering
    \includegraphics[width=\textwidth]{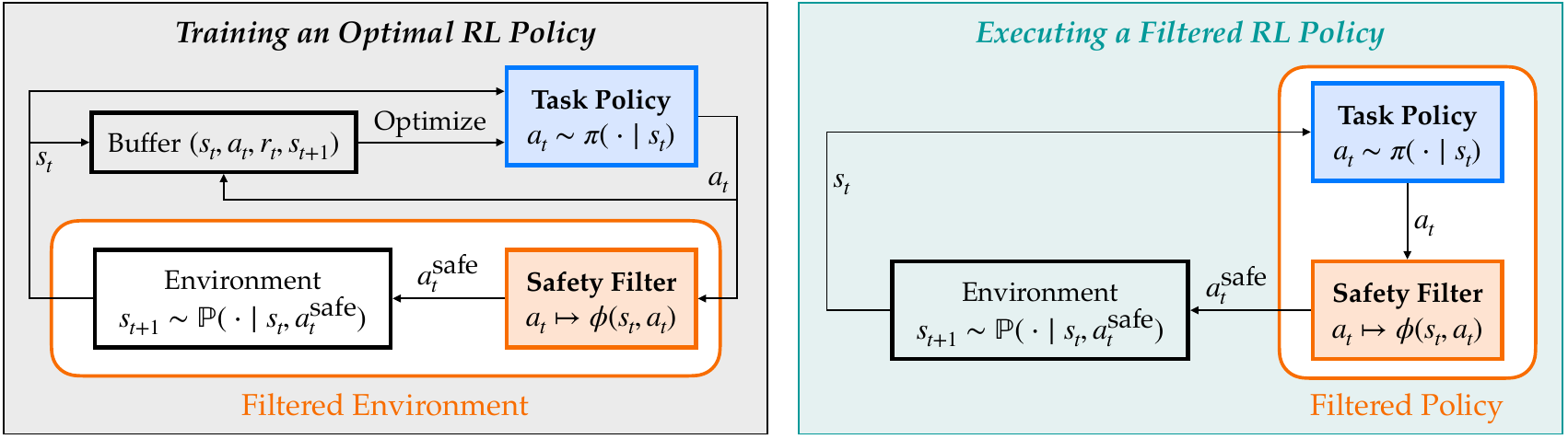}
    \caption{Our proposed framework for training and deploying optimal \gls{RL} policies under safety filtering. First, we train a filter-agnostic task policy in a safety-enforcing environment where every action is passed through a safety filter.
    Then, we deploy the learned task policy with the same filter at runtime. The filtered \gls{RL} policy provably achieves the \emph{same asymptotic performance} as if it had been trained with the safety constraints in mind.}
    \label{fig:control_diagram}
    \vspace{-3mm}
\end{figure}

\section{Background and Related Work}
\label{sec:rel_work}

\paragraph{Constrained reinforcement learning.}
Conventionally, \gls{RL} methods 
model safety 
through
a \gls{CMDP}~\cite{altman1999constrained}, where the agent aims to maximize its return subject to a budget on \emph{expected cumulative costs} incurred in failure states. 
Prominent \gls{CMDP}-based RL approaches
include trust-region methods such as \gls{CPO}~\cite{achiam2017constrained}, Lagrangian/primal--dual methods (including PID–Lagrangian stabilizations)~\cite{stooke2020responsive}, reward–penalty variants such as \gls{RCPO}~\cite{tessler2018reward}, and 
Lyapunov-based 
projected
policy optimization%
~\cite{chow2018lyapunov}.
Some methods depart from the standard \gls{CMDP} formulation
by imposing a cap on expected cost at each individual time step~\cite{li2021augmented}
or considering risk measures like \gls{CVaR}
rather than expected costs~\cite{chow2018risk}.
While effective at 
enforcing
some \emph{statistical measure} of safety,
the above approaches generally \emph{permit failures} along individual trajectories.
By contrast, our formulation considers truly safety-critical systems where catastrophic failures are \textit{categorically} unacceptable, 
and it consequently requires
their probability to be exactly zero.\footnote{%
Our problem can be viewed as a non-generic \textit{limiting class} of \glspl{CMDP} where the allowable budget is exactly zero, or costs are infinite. The \emph{opposite} limiting class is \glspl{MDP}, where the budget is infinite or costs are zero.
}
\looseness=-1

\paragraph{Safety filters.}
A safety filter is an automatic process that 
continually monitors an agent's proposed actions and, when necessary, modifies them to avoid entering no-win scenarios from which future failures may be inescapable.
These techniques are popular in robotics and control, where safety is
more often formalized as all-time satisfaction of state constraints~\cite{hsu2024safety}.
Representative classes include (i) \glspl{MPSF}, which verify system safety at runtime by forward-simulating a dynamics model or solving a trajectory optimization problem~\cite{wabersich2018linear, wabersich2021predictive, bastani2021safe}, (ii) \gls{CBF}-based filters, which 
perform runtime optimization 
at each control cycle to 
smoothly slow down any approach to the boundary of a known safe set~\cite{ames2016control, ames2019control, robey2020learning, oh2023safety}, and (iii) reachability-based filters, which first solve 
a Bellman/\gls{HJ} equation to compute the 
largest possible safe set
and then keep the agent from exiting it
at runtime~\cite{mitchell2005time, bansal2017hamilton}. 
Recently,
neural approximations of 
the \textit{safety value function}~\cite{fisac2019bridging, bansal2021deepreach, hsu2023isaacs}
have enabled scalable synthesis and deployment of reachability-based filters
beyond the reach of numerical tools, from robotic walking and manipulation to air traffic control and high-speed car racing~\cite{nguyen2025gameplay, wang2024magics, nakamura2025generalizing, choi_aloor_li_2025safemarl, oh2025safety}.

\paragraph{Safety-filtered reinforcement learning.} In safety-critical settings, safety filters can be used during and after \gls{RL} training to detect and modify unsafe candidate actions before the agent executes them in the environment.
Although related, safety-filtered \gls{RL} differs from shielded \gls{RL}~\cite{alshiekh2018safe,carr2023safe}: shields are synthesized over (finite-state) abstractions, whereas safety filters are defined directly on the original (often continuous-state) \gls{MDP} via controlled invariance.
Prior works in safety-filtered \gls{RL} have adopted \gls{CBF} safety filters not only to guarantee safety but also to guide learning by constraining the set of explorable policies~\cite{cheng2019end}; this idea also extends to robust \glspl{CBF} that handle model uncertainty and disturbances~\cite{emam2022safe}.
\gls{MPSF} enforces constraints via short-horizon prediction during learning, enabling zero or minimal violations in data collection and improved sample efficiency~\cite{wabersich2021predictive,bastani2021safe, bejarano2024safety}.
Reachability-based approaches include iteratively estimating the safe set (and sometimes a fallback) to override unsafe proposed actions online~\cite{ganai2023iterative}, as well as provably safe action projection using reachability analysis and polynomial zonotopes during training~\cite{kochdumper2023provably}.
Across these lines, empirical results report that \gls{RL} under safety filtering improves exploration efficiency by preventing safety violations and thus \emph{accelerates} the convergence of task policies~\cite{bejarano2024safety}.
However, these works ultimately fall short of providing \emph{optimality} guarantees for the executed (filtered) policy with respect to the safety-critical \gls{MDP}.
Recent works study optimality of \gls{RL} under a safety monitor~\cite{hunt2021verifiably} or a safeguard~\cite{markgraf2026safe}, concepts similar to safety filters; however, they do not characterize the \emph{maximal} controlled-invariant safe set, which prevents them from establishing that their safe \gls{RL} algorithms result in optimal asymptotic performance given a safety-critical \gls{MDP}.
We suspect it is this lack of a formal optimality guarantee that contributes to the common misconception that safety filtering hinders the attainable task performance of \gls{RL} policies~\cite{gros2020safe, anderson2020neurosymbolic, yang2023safe, koller2019learning}. 
Our results fill this gap by showing that, under a sufficiently permissive safety filter, the optimality of asymptotic task performance is preserved.

\paragraph{Safe model-based learning.}
Orthogonal to safety-critical model-free \gls{RL}, a rich line of literature focuses on learning (or refining) system \emph{dynamics} models and certifying safety during training model-based \gls{RL} policies.  
SafeOpt~\cite{sui2015safe} ensures safety while optimizing an unknown function modeled as a \gls{GP}, using confidence bounds to assess the safety of unexplored decisions.
In a complementary direction, Akametalu et al.~\cite{akametalu2014reachability} propose to reduce safety filter conservativeness by learning the system's unknown dynamics, treated as a disturbance input, with a \gls{GP}, and integrating the relaxed safety constraints into \gls{RL}.  
Similarly, Berkenkamp et al.~\cite{berkenkamp2017safe} and Richards et al.~\cite{richards2018lyapunov} learn Lyapunov certificates (and associated regions of attraction) from data and leverage them to enforce safe exploration.
Another line of work uses \gls{MPC} as the safety mechanism around learned dynamics—typically via (i) uncertainty-aware prediction for tightened or chance constraints, (ii) terminal conditions (invariant/robust sets and backups) for recursive feasibility, and (iii) safety certificates (Lyapunov/barrier) to justify constraint satisfaction; see Hewing et al.~\cite{hewing2020learning} for a comprehensive review.
Koller et al.~\cite{koller2019learning} couple a \gls{GP} dynamics model with chance-constrained \gls{MPC} to enable high-probability safe exploration while improving the policy.
In summary, these safe model-based learning methods generally require a coarse dynamic model and an initial safe set with a backup controller, rely on \emph{high-probability} (not categorical) guarantees, face scalability limits due to \gls{GP} regressors, and the learned dynamics are often residual on top of a known structure. 
By contrast, we focus on safe model-free reinforcement learning, which bypasses these modeling requirements and directly learns task policies under safety filtering without relying on explicit dynamics models.

\section{Problem Formulation}
\label{sec:theory}
We study an RL agent operating in a stochastic environment subject to failure conditions that must never occur.
We formalize this setting in two complementary ways: (i) as a \gls{SC-MDP}, where only actions that keep the agent safe are allowed, and (ii) as a filtered \gls{MDP}, where a safety filter minimally intervenes to correct unsafe actions. 
These two perspectives lay the groundwork for the theoretical result in the next section: an optimal policy learned in the filtered \gls{MDP}, and subsequently deployed with the same filter, achieves the same asymptotic performance as the best policy in the \gls{SC-MDP}.
The assumptions we adopt in this section to formulate the safety-critical decision-making problem and prove our theoretical results are introduced in \autoref{ass:standing_formal}.

\subsection{Safety-Critical Markov Decision Process}
\label{subsec:problem_formulation}
We consider an \gls{MDP} defined by tuple $\mathcal{M}=(\mathcal{S}, \mathcal{A}, \mathbb{P}, r, \gamma)$, where $\mathcal{S}$ is the state space, $\mathcal{A}$ is the action space, $\mathbb{P}:\mathcal{S}\times\mathcal{A}\rightarrow \mathcal{P}(\mathcal{S})$ is the transition probability, $r:\mathcal{S}\times\mathcal{A}\rightarrow\mathbb{R}$ is the reward function, and $\gamma\in(0,1)$ is the discount factor. 
We denote $\mathcal{F}\subset\mathcal{S}$ as the \emph{failure set} encoding
conditions deemed unacceptable and must be strictly avoided at all times, formally defined as:
\begin{equation}
    \mathcal{F}\coloneqq\bigl\{s\in\mathcal{S}:\ g(s)<0\bigr\},
\end{equation}
where $g$ is a safety \textit{margin} function representing $\failure$ as its zero sublevel set.

\begin{definition}[Safety-Critical \gls{MDP} (\gls{SC-MDP})]
\label{def:SC-MDP}
An \emph{\gls{SC-MDP}} is a tuple $\mathcal{M}_\mathrm{SC}=(\mathcal{S},\mathcal{A},\mathbb{P},r,\gamma, \mathcal{F})$.
For each $s\in\mathcal{S}$, $\mathcal{M}_\mathrm{SC}$ defines an infinite-horizon constrained optimal control problem:
\begin{subequations}
\label{eq:scmdp-opt}
\begin{align}
\max_{\pi}\quad & 
\mathbb{E}\!\left[\sum_{t=0}^{\infty}\gamma^{t} r(s_t,a_t)\ \middle|\ s_0=s,\ a_t\sim\pi(\cdot\mid s_t),\ s_{t+1}\sim\mathbb{P}(\cdot\mid s_t,a_t)\right] \\
\st\quad &
\Pr\!\left[s_t\notin\mathcal{F},\,\forall t\ \middle|\ s_0=s\right]=1, \label{eq:all_time_safe}
\end{align}
\end{subequations}
where the task policy $\pi$ is a stochastic kernel on $\mathcal{A}$ given $\mathcal{S}$.
\end{definition}

We now characterize what constitutes an admissible policy,
\ie, one that satisfies the all-time safety constraint~\eqref{eq:all_time_safe},
by adopting the notion of set invariance~\cite{blanchini1999set}. 
Specifically, a \emph{controlled-invariant} set $\Omega\subset\mathcal{S}$
guarantees the existence of a policy $\pi$ that recursively keeps the state $s$ within $\Omega$ for all time, \ie, $\forall s\in\Omega$ and $\forall a\in\support\pi(\cdot\mid s)$, $\support\mathbb{P}(\cdot\mid s, a)\subseteq\Omega$. 
Denoting $\Omega^*\subseteq\mathcal{S}\setminus\mathcal{F}$ as the \emph{maximal} controlled-invariant safe set, for each $s\in\Omega^*$, we define the \textit{safe action set} as
\begin{equation}
\mathcal{A}_{\mathrm{safe}}(s)\coloneqq\{\,a\in\mathcal{A}:\ \support\mathbb{P}(\cdot\mid s,a)\subseteq\Omega^*\,\}.
\end{equation}
Now, we define the \emph{admissible policy set} as
\begin{equation}
\Pi_{\mathrm{safe}}
\coloneqq
\Bigl\{\pi:\ \support\pi(\cdot\mid s)\subseteq \mathcal A_{\mathrm{safe}}(s),\, \forall s\in\Omega^* \Bigr\}.
\end{equation}
The admissible policy set $\Pi_\mathrm{safe}$ is the largest set of policies that ensure all-time safety~\eqref{eq:all_time_safe}
provided that the system is initialized within $\Omega^*$ 
(see \autoref{prop:maximal-admissible} for the proof).
For an admissible policy $\pi\in\Pi_{\mathrm{safe}}$, we define its \textit{value} on $\mathcal{M}_\mathrm{SC}$ as
\begin{equation}
V^{\pi}_{\mathcal M_{\mathrm{SC}}}(s)\coloneqq\mathbb E\!\left[\sum_{t=0}^{\infty}\gamma^{t}r(s_t,a_t)\ \middle|\ s_0=s,\ a_t\sim\pi(\cdot\mid s_t),\ s_{t+1}\sim\mathbb P(\cdot\mid s_t,a_t)\right],
\end{equation}
and the optimal value as $V^{*}_{\mathcal M_{\mathrm{SC}}}(s):=\sup_{\pi\in\Pi_{\mathrm{safe}}}V^{\pi}_{\mathcal M_{\mathrm{SC}}}(s)$.

\begin{remark}[Probabilistic and worst-case uncertainty]
\label{rem:SC-MDP}
The \gls{SC-MDP} unifies performance-centric \emph{probabilistic} uncertainty
and safety-centric \emph{deterministic} uncertainty
within a single framework. In contrast to conventional viewpoints that treat probabilistic and deterministic uncertainty as mutually exclusive, we show that combining them is not only theoretically sound but also advantageous: it achieves the best of both worlds---optimal expected task performance under strict safety guarantees.
\end{remark}

Real systems are often designed around \emph{unknown-but-bounded} uncertainty (\eg, external disturbances and modeling errors), while facing hazards captured here by the failure set $\mathcal{F}$. Many ``safe \gls{RL}'' benchmarks and baselines adopt the \gls{CMDP} formalism~\cite{ji2023safety,OpenAI2019SafetyGym}, in which the agent maximizes expected return \emph{subject to an expected cumulative constraint budget}. In that view, transient violations are permitted as long as their expected cumulative constraint cost remains below the threshold. As a result, \emph{\gls{CMDP} formulations can still permit rare but catastrophic failures} on individual trajectories, which are unacceptable in safety-critical decision-making. This also conflicts with the spirit of the \gls{ODD}, which specifies conditions under which a system should \emph{never} experience a specified class of failures~\cite{hsu2024safety}; \gls{CMDP} formulations generally do not provide per-trajectory, failure-free guarantees under such conditions.

By contrast, the \gls{SC-MDP} requires \emph{categorical} avoidance of failure, yielding a zero-violation specification by construction. It is robust to tail events while still leaving task performance optimized in expectation. This allows engineers to delineate a clear operating envelope within which the system is \emph{guaranteed} to operate safely—supporting deployment, auditability, and public trust in automated decision-making.

\subsection{Filtered Markov Decision Process}
\label{subsec:safety_filters}
To ensure that \gls{SC-MDP} satisfies all-time safety constraint~\eqref{eq:all_time_safe}, we leverage a \emph{safety filter}—an automatic process that monitors the agent's decision-making and, when necessary, 
modifies the nominal action in order to prevent safety violations. However, constraining the agent to an overly restrictive invariant set may hinder the agent from achieving satisfactory task performance, \ie, to accumulate high discounted rewards. Given the \gls{MDP} $\mathcal{M}$ together with the maximal controlled-invariant safe set $\Omega^*$, we formalize a \emph{perfect (least-restrictive) safety filter}~\cite{hsu2024safety} that enforces safety by intervening if and only if the nominal action causes state $\state$ to immediately exit $\Omega^*$.
\begin{definition}[Perfect safety filter]
\label{def:perfect_sf}
A \emph{perfect (least-restrictive) safety filter} is a map $\phi:\mathcal{S}\times\mathcal{A}\to\mathcal{A}$ such that, $\forall s\in\Omega^*$,
\begin{enumerate}
\item $\phi(s,a)\in\mathcal{A}_{\mathrm{safe}}(s)$, $\forall a\in\mathcal{A}$,
\item $\phi(s,a)=a$, $\forall a\in\mathcal{A}_{\mathrm{safe}}(s)$.
\end{enumerate}
\end{definition}

For safe \gls{RL}, we now define the \emph{filtered MDP}, where every action is passed through a perfect safety filter before execution. This can be interpreted as a ``bubble-wrapped'' environment, where the agent remains agnostic to the safety filter.

\begin{definition}[Filtered \gls{MDP}]
\label{def:filtered_mdp}
The filtered \gls{MDP} is a tuple $\mathcal{M}_{\phi}:=(\Omega^*,\mathcal{A},\mathbb{P}_{\phi},r_{\phi},\gamma)$, where
$\phi$ is a perfect safety filter, $\mathbb{P}_{\phi}(\cdot\!\mid s,a):=\mathbb{P}\bigl(\cdot\!\mid s,\phi(s,a)\bigr)$ is the filtered transition probability, and $r_{\phi}(s,a):=r\bigl(s,\phi(s,a)\bigr)$ is the filtered reward.
For a measurable and stationary policy $\pi$, we define its value on $\mathcal{M}_\phi$ as
\begin{equation}
    V^{\pi}_{\mathcal{M}_\phi}(s)
    :=\mathbb E\!\left[\sum_{t=0}^{\infty}\gamma^{t}\,r_{\phi}(s_t,a_t)\ \middle|\ s_0=s,\
    a_t\sim\pi(\cdot\mid s_t),\ s_{t+1}\sim\mathbb P_{\phi}(\cdot\mid s_t,a_t)\right],
\end{equation}
and the optimal filtered value as $V^{*}_{\mathcal{M}_\phi}(s):=\sup_{\pi}V^{\pi}_{\mathcal{M}_\phi}(s)$.
\end{definition}

\section{Safe Reinforcement Learning Theory}
In this section, we present our core theoretical results. First, we show that learning in the filtered \gls{MDP} $\mathcal{M}_\phi$ is provably safe, and that the standard RL convergence guarantees remain intact. Second, we prove that any policy that is optimal in $\mathcal{M}_\phi$, when executed under the same filter at deployment, is also optimal in the \gls{SC-MDP} $\mathcal{M}_\mathrm{SC}$.
To establish these results rigorously, we begin by stating the technical assumptions required for our analysis.

\begin{assumption}[Standing assumptions]
\label{ass:standing_formal}
We assume the following throughout this section:
\begin{enumerate}
\item (\emph{Spaces and measurability}) $\mathcal{S}$ and $\mathcal{A}$ are Borel spaces. $\Omega^*$ is nonempty and closed. $\mathbb{P}(\cdot\mid s,a)$ and $r(s,a)$ are Borel-measurable in their arguments.
\item (\emph{Safe initialization}) All training episodes start in $\Omega^*$.
\item (\emph{Stationarity and boundedness}) $\mathbb{P}(\cdot\mid s,a)$ and $r(s,a)$ are time-invariant. Rewards are bounded, \ie, $\sup_{s,a}|r(s,a)|<\infty$.
\item (\emph{Safety filter existence}) A perfect safety filter $\phi$ (\autoref{def:perfect_sf}) exists; $\phi$ is measurable and time-invariant.

\end{enumerate}
\end{assumption}

We are now ready to state our main theoretical result that formalizes the safety--performance separation principle in safe reinforcement learning.

\begin{theorem}[Safe and optimal \gls{RL} under a perfect safety filter]
\label{thm:safe_and_optimal_RL}
If~\autoref{ass:standing_formal} holds, then the following claims on the safety, convergence, and optimality of \gls{RL} under safety filtering are true:
\begin{enumerate}
\item \textbf{Safe learning.}
For any sequence of task policies produced by any RL algorithm during training, the filtered trajectories remain in $\Omega^*$ for all time.
\item \textbf{Convergence.}
Let $\textsc{Alg}$ be an \gls{RL} algorithm that converges on stationary discounted \glspl{MDP} with bounded rewards. Then, for any such \gls{MDP} $\mathcal{M}$, $\textsc{Alg}$ also converges on the corresponding filtered \gls{MDP} $\mathcal{M}_\phi$.
\item \textbf{Optimality under safety filtering.}
Let $\pi^\varepsilon_\phi$ denote a measurable and stationary $\varepsilon$-optimal policy on $\mathcal M_\phi$, possibly returned by $\textsc{Alg}$, for some $\varepsilon>0$:
\begin{equation}
    V^{\pi^\varepsilon_\phi}_{\mathcal{M}_\phi}(s)\ge V^*_{\mathcal{M}_\phi}(s)-\varepsilon,\qquad \forall s\in\Omega^*.
\end{equation}
We define the \emph{executed policy} $\pi_\mathrm{exec}$ as the pushforward of $\pi^\varepsilon_\phi$ by the map $a\mapsto\phi(s, a)$; \ie, for all Borel measurable sets $B\subseteq\mathcal{A}$,
\begin{equation}
\label{eq:executed_policy}
    \pi_\mathrm{exec}(B\mid s)\coloneqq\pi^\varepsilon_\phi\bigl(\{a\in\mathcal{A}:\ \phi(s, a)\in B\}\mid s\bigr),\qquad \forall s\in\Omega^*.
\end{equation}
Then, $\pi_\mathrm{exec}$ is a safe $\varepsilon$-optimal policy on $\mathcal{M}_\mathrm{SC}$:
\begin{equation}
    V^{\pi_\mathrm{exec}}_{\mathcal{M}_{\mathrm{SC}}}(s)\ \ge\ V^*_{\mathcal{M}_{\mathrm{SC}}}(s)-\varepsilon,\qquad \forall s\in\Omega^*.
\end{equation}
\end{enumerate}
\end{theorem}

\begin{proof}
The proof is deferred to \autoref{app:theory}.
\end{proof}

\begin{remark}[Safety with optimal performance]
\label{rem:safe_optimality}
\textbf{What \autoref{thm:safe_and_optimal_RL} establishes.}
For any $\varepsilon>0$, any $\varepsilon$-optimal policy on the filtered \gls{MDP} $\mathcal M_\phi$ induces, when executed through a perfect safety filter, a \emph{safe $\varepsilon$-optimal} policy on the SC-MDP $\mathcal M_{\mathrm{SC}}$. Consequently, as an \gls{RL} algorithm drives suboptimality on $\mathcal M_\phi$ to zero (e.g., by sufficient exploration of state–action pairs and appropriate stepsize conditions), the \emph{executed policy becomes asymptotically optimal} on $\mathcal M_{\mathrm{SC}}$:
\[
\lim_{\varepsilon \to 0^+} \bigl(V^*_{\mathcal M_{\mathrm{SC}}}(s)
- V^{\pi_{\mathrm{exec}}}_{\mathcal M_{\mathrm{SC}}}(s)\bigr)=0,\quad \forall s\in\Omega^*.
\]
Therefore, the \emph{safety–performance separation in \gls{RL} is complete}: enforcing safety with a sufficiently permissive safety filter during task policy training does not degrade asymptotic performance.
\end{remark}

\begin{remark}[Practical recipe for Safe \gls{RL}]
\label{rem:practical_recipe}
\textbf{Equivalence.}
Consider two ways to obtain a safe policy:

\begin{itemize}
  \item \emph{Oracle safe RL benchmark:} In the \gls{SC-MDP} ($\mathcal{M}_\mathrm{SC}$), choose any policy that is optimal among those that never violate safety (\ie, those in $\Pi_\mathrm{safe}$).
  \item \emph{Filter-based safe RL:} Learn in a world where every \emph{proposed} action is passed through the safety filter ($\mathcal{M}_\phi$), and pick any policy that is optimal there. At deployment in the \gls{SC-MDP}, run the learned policy through the same filter.
\end{itemize}
The two routines achieve the same long-run return on safe trajectories in the limit as the training suboptimality on $\mathcal{M}_\phi$ vanishes (\autoref{rem:safe_optimality}).

\textbf{Practical Recipe.} Unlike the oracle safe \gls{RL}, which cannot be realized by any practical algorithm, filter-based safe RL offers an actionable framework as depicted in \autoref{fig:control_diagram}:
first synthesize a (task-agnostic) safety filter for the \gls{SC-MDP}, then
train your task policy with any standard RL algorithm while keeping the safety filter in the loop, and deploy the same policy–filter pair at runtime.
The safety filter needs to be computed only once
for any number of task policies, 
provided that the transition probability $\mathbb{P}$ and the failure set $\failure$ stay the same.
Crucially, under the perfect (least-restrictive) safety filter assumption, this approach yields formal safety guarantees during both training and deployment with asymptotic task performance \emph{identical} to that obtained by training directly under hard safety constraints. In practice, approximate yet sufficiently permissive filters can still dramatically reduce (or eliminate) safety violations during training while matching or even surpassing the performance of baseline constrained \gls{RL} methods (see \autoref{subsec:experiment_setup}).
\end{remark}

\section{Experiments}
\label{sec:experiments}

To empirically validate our theoretical results, we build on the Safety Gymnasium benchmark~\cite{ji2023safety}, a modern successor to OpenAI’s Safety Gym~\cite{OpenAI2019SafetyGym}.
Safety Gymnasium provides a unified reinforcement learning environment for evaluating different training algorithms, emphasizing safety-critical tasks with standardized metrics.
Importantly, it retains the original design philosophy of Safety Gym: testing whether agents can train an optimal task policy without sacrificing safety during training, with enhanced environment and task diversity, simulation fidelity, and compatibility with modern \gls{RL} training pipelines such as Stable-Baselines3~\cite{raffin2021stable}.
By situating our experiments in this benchmark, we ensure that our evaluation (i) directly validates the safety--performance separation as stated in \autoref{thm:safe_and_optimal_RL}, (ii) is grounded in a standardized and widely used testbed for constrained \gls{RL}, and (iii) provides consistent results across diverse tasks and environments.
We expect that our safe \gls{RL} framework will serve as a state-of-the-art baseline for future work on safe reinforcement learning; because it is directly deployable within the Safety Gymnasium benchmark, it can be used out of the box to conveniently evaluate and compare new methods.

\subsection{Experiment Setup}
\label{subsec:experiment_setup}
In this section, we introduce the setup of our experiments, including the agent, environment, safety filters used for training and deployment, and baseline comparisons.
Full details of the experimental setup, including model architecture, safety specifications, and filter implementation, can be found in \autoref{app:implementation}.

\begin{figure}[t]
    \centering
    \includegraphics[width=0.8\textwidth]{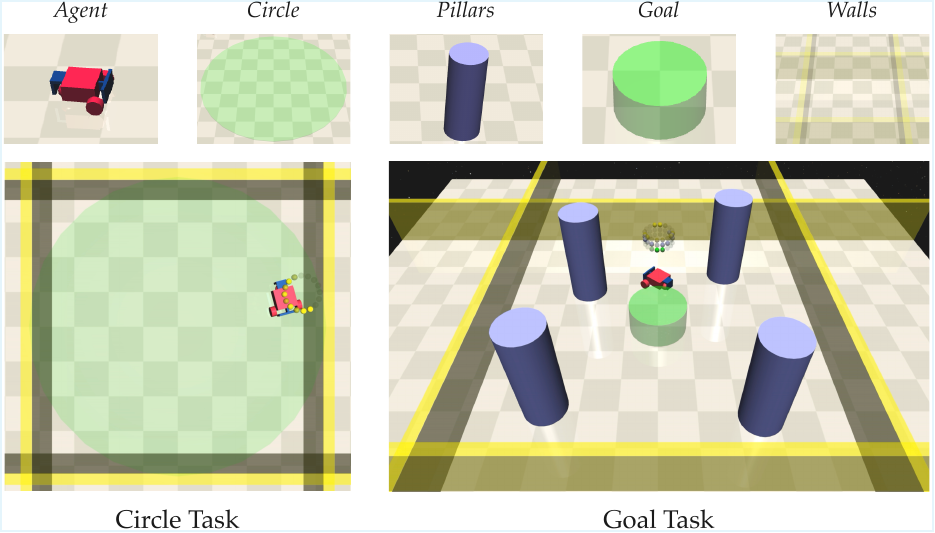}
    \caption{\textit{Top}: environment components (Agent, Circle, Pillars, Goal, Walls). \textit{Bottom}: example scenes for the two tasks from Safety Gymnasium~\cite{ji2023safety}. We use the \emph{Car} agent across all experiments. \emph{Circle} task: the agent aims to track the green circle's boundary as fast as possible while remaining inside the square wall; safety is violated by crossing the wall. \emph{Goal} task: the agent navigates to a fixed green goal while avoiding the pillars and staying inside the wall; safety is violated by crossing the wall or colliding with a pillar. In both tasks, the agent is randomly initialized inside the wall and away from the pillars and the goal.}
    \label{fig:environment}
    \vspace{-3mm}
\end{figure}

\textbf{Agent.} We use the \emph{Car} agent from Safety Gymnasium. It has continuous control inputs and is implemented as a differential-drive platform with two independently driven parallel wheels and a free-rolling rear wheel, exactly as provided in the benchmark suite~\cite{OpenAI2019SafetyGym,ji2023safety}.

\textbf{Tasks, rewards, and safety constraints.} 
We evaluate safe \gls{RL} with the following representative tasks from Safety Gymnasium~\cite{ji2023safety}, each with distinct reward functions and safety constraints.
\begin{itemize}
    \item \textbf{Goal:}
    The robot navigates to a fixed goal position while avoiding contact with cylindrical obstacles and remaining inside a square wall (see \autoref{fig:environment}). Specifically, the reward is proportional to the decrement of the distance between the robot and the goal over a single timestep, and an additional reward is provided when the robot reaches the goal. The robot is constrained within a square wall that contains four cylindrical obstacles (\ie, \emph{pillars}); the position of the square wall and the pillars remain fixed. Crossing the square wall or colliding with a pillar is considered a safety violation. The robot is randomly initialized inside the square wall, sufficiently far from the pillars and the goal.
    
    \item \textbf{Circle:}
    The robot aims to track the boundary of the green circle as fast as possible while remaining inside a square wall (see \autoref{fig:environment}). Specifically, the reward function consists of two factors: (i) a tracking reward that increases as the robot stays closer to the circle boundary, and (ii) a velocity factor that rewards larger tangential velocity along the circle. The robot is constrained within a square wall with a fixed position; crossing the wall is considered a safety violation. The robot is randomly initialized inside the square wall.
\end{itemize}

\textbf{Safety filters.}
We use two types of safety filters in our experiments: one with value-based monitoring~\cite[Section~3.1]{hsu2024safety} and the other with rollout-based monitoring~\cite[Section~3.3]{hsu2024safety}. For both filters, we learn a best-effort fallback policy and the associated safety value function with model-free, RL-based \gls{HJ} reachability analysis~\cite{fisac2019bridging, oh2025safety}. Specifically:
\begin{itemize}
    \item \textbf{Value-based filter}
    queries, at each timestep, the learned value function to determine whether the current state is safe. If the proposed action sampled from the task policy is deemed safe, we execute it; otherwise, we override it with the best-effort safety fallback.
    
    \item \textbf{Rollout-based filter} simulates, at each timestep, a state trajectory based on a proposed action sampled from the task policy, followed by a series of actions from the best-effort fallback policy for a fixed horizon.
    If the state reaches a known terminal safe state (or controlled-invariant safe set) within that horizon, we may execute an action sampled from the task policy; otherwise, we apply the safety fallback action.
\end{itemize}

While a valid (resolution-complete) safety filter with value-based monitoring can be obtained by solving the safety Bellman equation with \gls{DP}~\cite{bansal2017hamilton}, this method does not scale to more than 6 continuous state dimensions.
Therefore, we adopt RL-based reachability analysis to obtain \emph{neural approximations} of a perfect safety filter for our experiments, which is, to the best of our knowledge, the only purely \emph{model-free} method that \emph{scalably} approximates the \emph{maximal} controlled-invariant safe set and a fallback policy using a simulator.
This makes RL-based reachability analysis the best-suited safety filter synthesis method for Safety Gymnasium benchmark, which provides more than 40 continuous observation dimensions and no information regarding the system dynamics.
In \autoref{subsec:results}, we show empirically that despite the lack of formal guarantees, a value-based neural safety filter can significantly reduce the frequency of safety violations during training.
Furthermore, the rollout-based monitoring, unlike the value-based counterpart, leads to a provably \emph{valid safety filter} for the benchmarking tasks in Safety Gymnasium, since coming to a stop is a known terminal safe state for the \emph{Car} agent.
Indeed, as shown in \autoref{subsec:results}, this approach achieves \textit{zero safety violations}.
Moreover, the task policies trained under both filters converge to the same—and, in some cases, higher—episodic return when compared to the baselines.

\textbf{Baselines.}
We compare against two baselines for constrained \gls{RL}:
\begin{itemize}
    \item \textbf{\gls{CPO}~\cite{achiam2017constrained}:} The agent optimizes its return subject to a budget on the environment’s constraint costs. We use the same implementation provided by Safety Gymnasium exactly as released.
    \item \textbf{Standard \gls{SAC}:}
    The agent optimizes its return with no explicit constraints or costs other than episode termination, which occurs immediately after a safety violation. Because episode termination prevents the agent from accruing positive rewards in the future, this setup implicitly encourages the agent to learn policies that satisfy the safety constraints. This is a commonly used heuristic for \gls{RL} under soft safety constraints~\cite{oh2025safety, faust2018prm}.
\end{itemize}

\textbf{\gls{RL} algorithms.}
All task policies are trained with \gls{SAC}~\cite{haarnoja2018soft} (except the \gls{CPO} baseline) using the Stable-Baselines3~\cite{raffin2021stable} implementation. Training budgets (environment steps per run), model capacity, and evaluation protocols are kept identical across all methods to ensure a fair comparison.

\textbf{Reproducibility.}
Each method is run with \emph{five} independent random seeds. We report the mean and standard error across seeds under a fixed training step budget and a shared evaluation protocol (periodic evaluation episodes with safety violations and episodic returns logged).
The code and experimental scripts used to reproduce our results are publicly available in an online repository.\footnote{\url{https://github.com/SafeRoboticsLab/Provably-Optimal-RL}}

\subsection{Results}
\label{subsec:results}

\begin{figure}[t]
    \centering
    \includegraphics[width=\textwidth, trim=0 0 0 10, clip]{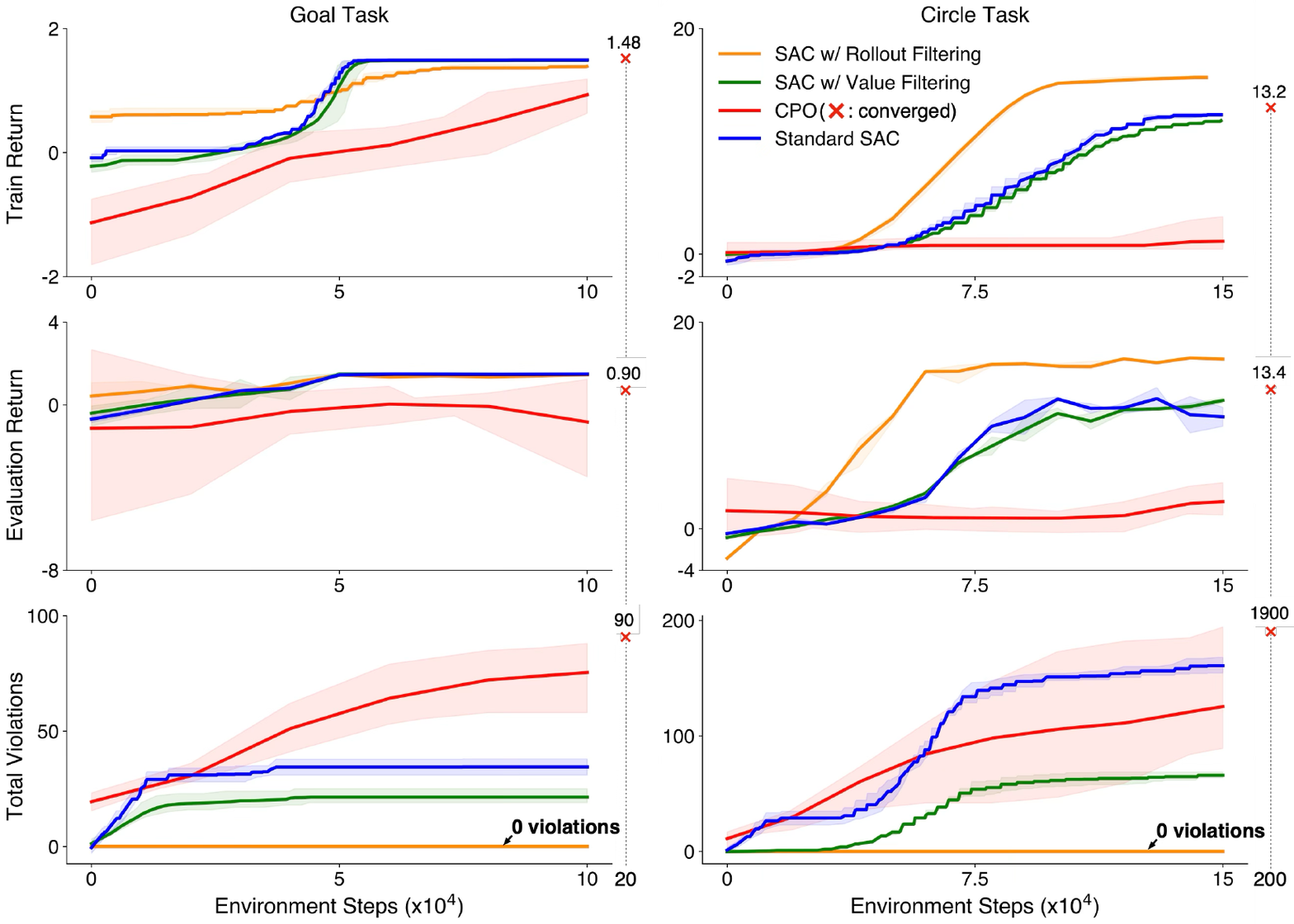}
    \caption{Experimental results on the \emph{Goal} (left) and \emph{Circle} (right) tasks in Safety Gymnasium. We compare our safe RL framework, implemented with an \gls{SAC}-based task policy and a neural safety filter with rollout- and value-based monitoring, against \gls{CPO} and standard \gls{SAC} baselines. \textit{Top}: mean episodic training return versus the number of environment steps. \textit{Middle}: mean episodic evaluation return versus the number of environment steps. \textit{Bottom}: cumulative safety violations versus the number of environment steps.
    Our method achieves \textbf{zero safety violations} when using a valid (rollout-based) safety filter while \textbf{converging to the same or higher return than the baselines}.
    Even with an error-prone approximate (value-based) filter, the number of safety violations is significantly reduced compared to the baselines.}
    \label{fig:experiment_results}
    \vspace{-3mm}
\end{figure}

In this section, we report results in training task policies under a value-based filter and a rollout-based filter, and compare them to the baselines.
Specifically, we focus on (i) safety violations during task policy training, (ii) convergence of the \gls{RL} algorithm, and (iii) optimality of the trained task policy.

\textbf{Safety during RL training.}
The number of safety violations accumulated up to a given environment step during task policy training is shown in the last row of \autoref{fig:experiment_results}.
Remarkably, our \emph{rollout-based filter recorded zero safety violations} in both tasks across five random seeds.
This is expected since the rollout-based filter is a \emph{valid safety filter} with strict guarantees, and it directly supports our theoretical claim of \emph{safe learning} in \autoref{thm:safe_and_optimal_RL}: given a valid safety filter, system trajectories remain safe throughout task policy training.
Although our value-based filter produced a non-zero violation count due to neural approximation error, it nevertheless resulted in fewer violations on both tasks compared to the two baselines.
Standard \gls{SAC} incurred the second-most violations on the \emph{Goal} task and the most on the \emph{Circle} task.
We also observe that the cumulative violation curves plateau with training, indicating that although standard \gls{SAC} cannot enforce safety during training, the learned task policy increasingly avoids violations as the algorithm converges.
\gls{CPO} recorded the most total violations on \emph{Goal} and the second-most on \emph{Circle}. For completeness, we additionally mark \gls{CPO}’s post-convergence metric (symbol ``\textcolor{red}{×}'' in \autoref{fig:experiment_results}), since \gls{CPO} converges much later than the other methods. Our primary comparisons, however, use counts at a common training budget—$100$k environment steps for \emph{Goal} and $150$k for \emph{Circle}—for all methods.

\textbf{Convergence of RL training under safety filtering.}
From \autoref{fig:experiment_results}, we conclude that \gls{SAC} under both rollout- and value-based filtering, as well as standard \gls{SAC}, converged well within our training budget.
For these three methods, the episodic training return and the evaluation return plateau within the budget on both tasks, and the variance of these metrics remains very small across seeds, indicating convergence over repeated runs.
This supports our theoretical claim of \emph{convergence} in \autoref{thm:safe_and_optimal_RL}: if an \gls{RL} algorithm converges on an (unconstrained) \gls{MDP}, then the same algorithm converges on its filtered counterpart.
By contrast, \gls{CPO} takes significantly more environment steps to converge, and its episodic returns exhibit substantially larger variance than the other filtered \gls{RL} methods.

In the \emph{Circle} task, we observe that training under the \emph{rollout‐based} filter converges substantially faster (in environment steps) than the other methods.
This acceleration is consistent with prior reports that safety filtering improves sample efficiency by pruning unsafe exploration~\cite{wabersich2021predictive,bastani2021safe,bejarano2024safety}.
A plausible explanation is task geometry and reward shaping: \emph{Circle} rewards encourage high tangential speed while staying close to the circle boundary, which lies inside the square wall.
Pushing for higher reward inevitably increases the chance of wall contact under unconstrained exploration; the rollout-based filter preemptively removes those unsafe proposals, concentrating data on productive (safe) trajectories and speeding convergence.
By contrast, the effect is less pronounced in the \emph{Goal} task.
Depending on the random initialization, a shortest path to the goal can often avoid obstacles and the wall even without interventions, so there is less unsafe exploration to prune. 
Consequently, while the safety filter still prevents violations, the gain in exploration efficiency—and thus in convergence rate—is smaller on \emph{Goal} than on \emph{Circle}.
Characterizing and optimizing the convergence rate for reinforcement learning under safety filtering is beyond the scope of this paper and remains an important direction for future work.

\textbf{Performance of the task policy.}
As shown in \autoref{fig:experiment_results}, on both tasks the policies trained under rollout‐ and value‐based filtering attain final performance that matches—or exceeds—standard \gls{SAC} and \gls{CPO}. This empirically supports our most significant claim of \emph{optimality under safety filtering} (\autoref{thm:safe_and_optimal_RL}, \autoref{rem:safe_optimality}): enforcing safety with a sufficiently permissive safety filter during training does not degrade the asymptotic performance.
In the \emph{Goal} task, training under rollout‐ and value‐based filtering yields nearly identical episodic training and evaluation returns, while \gls{CPO} achieves a substantially lower return under the same training budget. 
As a sanity check, we run \gls{CPO} until convergence and observe that its return eventually matches that of the other methods.
In the \emph{Circle} task, training with the rollout‐based filter achieves significantly higher return than both the value‐based filter and standard \gls{SAC}; \gls{CPO} records little meaningful return within budget and, even post‐convergence, reaches a level comparable to value‐based filtering and standard \gls{SAC}, yet still below the policy trained with the rollout-based filter.
Taken together, these results demonstrate that task policy training under safety filtering attains the same (and, in some cases, higher) task performance as unconstrained baselines, providing concrete empirical evidence for the \emph{safety–performance separation} in safe \gls{RL}.

\section{Conclusion}
In this paper, we have established formal convergence and optimality results for reinforcement learning under safety filtering (safe \gls{RL}).
Our theoretical analysis proves that, under an idealized safety filter that is minimally restrictive yet capable of preventing all unsafe actions, reinforcement learning achieves complete safety--performance separation, yielding the same asymptotic performance as direct optimization under hard safety constraints.
This result resolves a longstanding misconception in reinforcement learning that enforcing safety inevitably limits the agent's attainable performance.
It shows that safety filtering provides a principled mechanism for maintaining both formal safety guarantees and optimal long-term behavior in RL, independent of the specific RL algorithm used for optimizing task performance.
Empirical validation on Safety Gymnasium benchmarks further supports our theory, demonstrating that, in practice, the proposed safe RL framework achieves zero safety violations with a valid safety filter, while converging to a policy that matches or surpasses baseline performance.
Taken together, these findings provide a rigorous paradigm for safe reinforcement learning: learn your safety filter once, train any RL algorithm with the filter in the loop, and deploy the same filter--policy pair to achieve strict safety and optimal performance in tandem.

\begin{ack}
This work has been partially supported by the DARPA Learning Introspective Control (LINC) program (award number N6523623C8002) and by the NSF CAREER Award (number 2340851).
\end{ack}

\printbibliography

\newpage
\appendix

\section{Theoretical Results and Proofs}
\label{app:theory}
\begin{lemma}[No all-time safety outside the maximal invariant set]
\label{lem:maximal-invariant}
Let $\Omega^*\subseteq \mathcal S\setminus\mathcal F$ be the maximal controlled-invariant safe set.
Then, for any stationary policy $\pi$ and any $s\in {\Omega^{*}}^c\cap(\mathcal S\setminus\mathcal F)$,
\[
\Pr_{\pi,\mathbb P}\!\big[s_t\notin\mathcal F,\,\forall t\ \big|\ s_0=s\big]\;<\;1.
\]
In words, starting outside $\Omega^*$, no stationary policy can satisfy the all-time safety constraint with probability $1$.
\end{lemma}

\begin{proof}
We prove by contradiction. Assume there exists $s\in{\Omega^*}^c\cap(\mathcal{S}\setminus\mathcal{F})$ and a stationary policy $\pi$ such that $\Pr_{\pi, \mathbb{P}}[s_t\notin \mathcal{F},\,\forall t\mid s_0=s]=1$.
Define
\[
R_\pi\;:=\;\Big\{x\in \mathcal S\setminus\mathcal F:\ 
\Pr_{\pi,\mathbb P}\!\big[s_t\notin\mathcal F,\,\forall t\ \big|\ s_0=x\big]=1\Big\}.
\]
By the Markov property, if $x\in R_\pi$, then the set of next possible states rendered by $\pi$ and $\mathbb{P}$ should be a subset of $R_\pi$. More formally, for all $a\in\support\pi(\cdot\mid x)$, $\support\mathbb{P}(\cdot\mid x, a)\subseteq R_\pi$. Thus, $R_\pi$ is a controlled-invariant subset of $\mathcal{S}\setminus\mathcal{F}$ under $\pi$, and since there exists $s\in R_\pi\setminus\Omega^*$, the set $\Omega^*\cup R_\pi$ is a strictly larger controlled-invariant safe set than $\Omega^*$. This contradicts the maximality of $\Omega^*$. Therefore, no such $s$ exists.
\end{proof}

\begin{proposition}[Maximality of the admissible policy set]
\label{prop:maximal-admissible}
Let $\Omega^*\subseteq \mathcal S\setminus\mathcal F$ be the maximal controlled-invariant safe set, and define
$\mathcal A_{\mathrm{safe}}(s):=\{a\in\mathcal A:\support\mathbb P(\cdot\mid s,a)\subseteq \Omega^*\}$ for all $s\in\Omega^*$. 
Consider the set of measurable, stationary policies
\[
\Pi_{\mathrm{safe}}
\coloneqq
\Bigl\{\pi:\ \support\pi(\cdot\mid s)\subseteq \mathcal A_{\mathrm{safe}}(s),\, \forall s\in\Omega^* \Bigr\}.
\]
Then $\Pi_{\mathrm{safe}}$ is the largest collection of policies that satisfy the all-time safety constraint from every initial state in $\Omega^*$.
\end{proposition}

\begin{proof}
(\emph{$\Pi_{\mathrm{safe}}$ policies are admissible.})
Let $\pi\in\Pi_{\mathrm{safe}}$. For any $s\in\Omega^*$ and any $a\in\support\pi(\cdot\mid s)$, we have $\support\mathbb P(\cdot\mid s,a)\subseteq\Omega^*$. By induction, we get $\Pr_{\pi,\mathbb P}\!\big[s_t\notin\mathcal F,\,\forall t\ \big|\ s_0=s\big]=1$.

(\emph{Policies not in $\Pi_{\mathrm{safe}}$ are not admissible.})
Let $\pi\notin\Pi_{\mathrm{safe}}$. Then there exist $s\in\Omega^*$ and $a\in\support\pi(\cdot\mid s)$ such that $a\notin\mathcal A_{\mathrm{safe}}(s)$, so $\support\mathbb P(\cdot\mid s,a)\nsubseteq\Omega^*$. Hence, with positive probability, the next state rendered by $\pi$ and $\mathbb{P}$ lies in ${\Omega^*}^c$. By \autoref{lem:maximal-invariant}, the overall probability of ever entering $\mathcal F$ starting from $s$ under $\pi$ is positive. Thus $\pi$ violates the all-time safety constraint from $s$.

Maximality of $\Pi_{\mathrm{safe}}$ follows immediately: any stationary policy assigning positive mass to an action outside $\mathcal A_{\mathrm{safe}}(s)$ at some $s\in\Omega^*$ cannot satisfy the all-time safety constraint from that state.
\end{proof}

\begin{theorem*}[Restatement of \autoref{thm:safe_and_optimal_RL}]
\label{thm:safe_and_optimal_RL_appendix}
If~\autoref{ass:standing_formal} holds, then the following claims on the safety, convergence, and optimality of \gls{RL} under safety filtering are true:
\begin{enumerate}
\item \textbf{Safe learning.}
For any sequence of task policies produced by any RL algorithm during training, the filtered trajectories remain in $\Omega^*$ for all time.
\item \textbf{Convergence.}
Let $\textsc{Alg}$ be an \gls{RL} algorithm that converges on stationary discounted \glspl{MDP} with bounded rewards. Then, for any such \gls{MDP} $\mathcal{M}$, $\textsc{Alg}$ also converges on the corresponding filtered \gls{MDP} $\mathcal{M}_\phi$.
\item \textbf{Optimality under safety filtering.}
Let $\pi^\varepsilon_\phi$ denote a measurable and stationary $\varepsilon$-optimal policy on $\mathcal M_\phi$, possibly returned by $\textsc{Alg}$, for some $\varepsilon>0$:
\begin{equation*}
    V^{\pi^\varepsilon_\phi}_{\mathcal{M}_\phi}(s)\ge V^*_{\mathcal{M}_\phi}(s)-\varepsilon,\qquad \forall s\in\Omega^*.
\end{equation*}
We define the \emph{executed policy} $\pi_\mathrm{exec}$ as the pushforward of $\pi^\varepsilon_\phi$ by the map $a\mapsto\phi(s, a)$; \ie, for all Borel measurable sets $B\subseteq\mathcal{A}$,
\begin{equation*}
    \pi_\mathrm{exec}(B\mid s)\coloneqq\pi^\varepsilon_\phi\bigl(\{a\in\mathcal{A}:\ \phi(s, a)\in B\}\mid s\bigr),\qquad \forall s\in\Omega^*.
\end{equation*}
Then, $\pi_\mathrm{exec}$ is a safe $\varepsilon$-optimal policy on $\mathcal{M}_\mathrm{SC}$:
\begin{equation*}
    V^{\pi_\mathrm{exec}}_{\mathcal{M}_{\mathrm{SC}}}(s)\ \ge\ V^*_{\mathcal{M}_{\mathrm{SC}}}(s)-\varepsilon,\qquad \forall s\in\Omega^*.
\end{equation*}
\end{enumerate}
\end{theorem*}

\begin{proof}
(\emph{Safe learning})
Since all training episodes begin in $\Omega^*$ (\autoref{ass:standing_formal}) and $\phi(s, a)$ renders $\Omega^*$ invariant for all $a\in\mathcal{A}$ produced by any task policy (\autoref{def:perfect_sf}, \autoref{ass:standing_formal}), all filtered trajectories remain in $\Omega^*$ for all time.

(\emph{Convergence})
From~\autoref{ass:standing_formal} and \autoref{def:filtered_mdp}, both $\mathbb{P}_\phi$ and $r_\phi$ are time-invariant and $r_\phi$ is bounded. Thus, the filtered \gls{MDP} $\mathcal{M}_\phi$ is a stationary discounted \gls{MDP} with bounded rewards. The iterates of $\textsc{Alg}$ depend only on the executed tuples
\begin{equation*}
    (s_t, a_t, r_t, s_{t+1}) \quad \text{with} \quad r_t=r_\phi(s_t, a_t),\quad s_{t+1}\sim\mathbb{P}_\phi(\cdot\mid s_t, a_t),
\end{equation*}
and on the stepsize/exploration schedule of $\textsc{Alg}$. Every step in the convergence proof of $\textsc{Alg}$ on $\mathcal{M}$ applies to $\mathcal{M}_\phi$ with the substitution $(\mathbb{P}, r)\mapsto(\mathbb{P}_\phi, r_\phi)$. 
Therefore $\textsc{Alg}$ converges on $\mathcal{M}_\phi$.
This carry-over presumes the same algorithmic side conditions—e.g., stepsize conditions and ergodicity/coverage under the executed policy—also hold in $\mathcal{M}_\phi$. It covers, for example, almost-sure convergence of tabular Q-learning~\cite{watkins1992q,tsitsiklis1994asynchronous}, two-timescale actor–critic~\cite{konda1999actor}, and soft policy-iteration convergence for \gls{SAC}~\cite{haarnoja2018soft}.

(\emph{Optimality under safety filtering})
We begin with showing the equivalence between $V^*_{\mathcal{M}_\phi}$ and $V^*_{\mathcal{M}_\mathrm{SC}}$. We define the optimal Bellman operators for the filtered \gls{MDP} $\mathcal{M}_\phi$ and the \gls{SC-MDP} $\mathcal{M}_\mathrm{SC}$, respectively:
\begin{align*}
    (\mathcal{T}^*_{\mathcal{M}_\phi}V)(s)&=\sup_{a\in\mathcal{A}}\Bigl\{r_\phi\bigl(s, a\bigr)+\gamma\,\mathbb{E}_{s'\sim\mathbb{P}_\phi(\cdot\mid s,a)}[V(s')]\Bigr\}
    \\
    &=\sup_{a\in\mathcal{A}}\Bigl\{r\bigl(s,\phi(s,a)\bigr)+\gamma\,\mathbb{E}_{s'\sim\mathbb{P}(\cdot\mid s,\phi(s,a))}[V(s')]\Bigr\}, \\
    (\mathcal{T}^*_{\mathcal{M}_\mathrm{SC}}V)(s)&=\sup_{a\in\mathcal{A}_{\mathrm{safe}}(s)}\Bigl\{r(s,a)+\gamma\,\mathbb{E}_{s'\sim\mathbb{P}(\cdot\mid s,a)}[V(s')]\Bigr\}.
\end{align*}
From \autoref{def:perfect_sf}, the image of the map $a\mapsto \phi(s,a)$ is exactly $\mathcal{A}_{\mathrm{safe}}(s)$ and $\phi$ is the identity on that image. Taking the supremum over $a\in\mathcal{A}$ post-composition by $\phi$ equals the supremum over $a\in\mathcal{A}_{\mathrm{safe}}(s)$. Therefore, we have
\begin{equation*}
    \mathcal{T}^*_{\mathcal{M}_\phi}=\mathcal{T}^*_{\mathcal{M}_\mathrm{SC}},\qquad\text{pointwise on }\Omega^*.
\end{equation*}
Since both operators are $\gamma$-contractions on $(\mathcal B_b(\Omega^*),\|\cdot\|_\infty)$—the Banach space of bounded Borel-measurable value functions $V:\Omega^*\to\mathbb R$ endowed with the sup norm—they have the same unique fixed point:
\begin{equation*}
    V^*_{\mathcal{M}_\phi}(s)=V^*_{\mathcal{M}_\mathrm{SC}}(s),\qquad \forall s\in\Omega^*.
\end{equation*}
We now show the equivalence between $V^{\pi^\varepsilon_\phi}_{\mathcal{M}_\phi}$ and $V^{\pi_{\mathrm{exec}}}_{\mathcal{M}_{\mathrm{SC}}}$. We unroll the expectation and apply the definition of the filtered \gls{MDP} $\mathcal{M}_\phi$ (\autoref{def:filtered_mdp}):
\begin{align*}
    V^{\pi^\varepsilon_\phi}_{\mathcal{M}_\phi}(s)&=\mathbb E\!\left[\sum_{t=0}^{\infty}\gamma^{t}\,r_{\phi}(s_t,a_t)\ \middle|\ s_0=s,\ a_t\sim\pi^\varepsilon_\phi(\cdot\mid s_t),\ s_{t+1}\sim\mathbb P_{\phi}(\cdot\mid s_t,a_t)\right]\\
    &=\mathbb E\!\left[\sum_{t=0}^{\infty}\gamma^{t}\,r\bigl(s_t,\phi(s_t, a_t)\bigr)\ \middle|\ s_0=s,\ a_t\sim\pi^\varepsilon_\phi(\cdot\mid s_t),\ s_{t+1}\sim\mathbb P\bigl(\cdot\mid s_t,\phi(s_t, a_t)\bigr)\right].
\end{align*}
Let $\tilde{a}_t\coloneqq\phi(s_t, a_t)$, where $a_t\sim\pi^\varepsilon_\phi(\cdot\mid s_t)$. The executed policy $\pi_\mathrm{exec}$ is defined as the pushforward of the $\varepsilon$-optimal policy $\pi^\varepsilon_\phi$ by the map $a\mapsto\phi(s, a)$. Therefore, $\tilde{a}_t\sim\pi_\mathrm{exec}(\cdot\mid s_t)$, and we have 
\begin{align*}
    V^{\pi^\varepsilon_\phi}_{\mathcal{M}_\phi}(s)&=\mathbb E\!\left[\sum_{t=0}^{\infty}\gamma^{t}\,r(s_t,\tilde{a}_t)\ \middle|\ s_0=s,\ \tilde{a}_t\sim\pi_\mathrm{exec}(\cdot\mid s_t),\ s_{t+1}\sim\mathbb P(\cdot\mid s_t,\tilde{a}_t)\right]\\
    &=V^{\pi_{\mathrm{exec}}}_{\mathcal{M}}(s),\qquad \forall s\in\Omega^*.
\end{align*}
By the perfect safety filter property (\autoref{def:perfect_sf}) and the pushforward definition of $\pi_\mathrm{exec}$, we have $\support\pi_\mathrm{exec}(\cdot\mid s)\subseteq\mathcal{A}_\mathrm{safe}(s)$ for all $s\in\Omega^*$. Therefore, $\pi_\mathrm{exec}\in\Pi_\mathrm{safe}$, and this gives
\[
V^{\pi_{\mathrm{exec}}}_{\mathcal{M}}(s)=V^{\pi_{\mathrm{exec}}}_{\mathcal{M}_{\mathrm{SC}}}(s),\qquad \forall s\in\Omega^*.
\]
Finally, combining the equalities yields
\[
V^{\pi_{\mathrm{exec}}}_{\mathcal{M}_{\mathrm{SC}}}(s)
=
V^{\pi^\varepsilon_\phi}_{\mathcal{M}_\phi}(s)
\ \ge\
V^*_{\mathcal{M}_\phi}(s)-\varepsilon
=
V^*_{\mathcal{M}_{\mathrm{SC}}}(s)-\varepsilon,\qquad \forall s\in\Omega^*.
\]
This proves the $\varepsilon$-optimality and the safety of $\pi_\mathrm{exec}$ on $\mathcal{M}_\mathrm{SC}$. 
\end{proof}

\section{Implementation Details} \label{app:implementation}
\p{Model architecture}
The SAC task policy and safety policy have the same architecture, with the actor and critic policies implemented by a fully connected feedforward neural network with 2 hidden layers of $256$ neurons. The CPO task policy has 2 hidden layers of $64$ neurons. All policies use \emph{ReLU} activations for SAC-based architectures and \emph{Tanh} activations for CPO.

The safety policy is trained with a learning rate $1\times10^{-5}$, replay buffer size of $2\times10^5$, batch size of $256$, discount factor $\gamma = 0.995$, and soft update coefficient $\tau=0.01$, for a total of $2\times10^6$ steps. 

The task policy SAC is trained with a learning rate $3\times10^{-4}$, replay buffer size of $1\times10^5$, batch size of $256$, $\gamma=0.99$, and $\tau=0.01$. 

The CPO policy is trained with a learning rate of $3\times10^{-4}$, $\gamma=0.99$, and KL-divergence step size of $\delta_\text{KL}=0.01$.

\p{Safety filter implementation}
For rollout filtering, we adopt the rollout evaluation procedure described in Nguyen et al.~\cite{nguyen2025gameplay}. We use a finite rollout horizon of $H=100$ and define the target margin function $l(s)$ and stop policy $\pi_\text{stop}$ as

\begin{align*}
l(s) &= \eta - \sqrt{v_x^2 + v_y^2}, \qquad 
\pi_{\text{stop}}(a \mid s) = \mathbf{0}
\end{align*}

where $\eta$ denotes the target safety velocity threshold. Although the safety policy was not explicitly trained to bring the robot to a complete stop, it implicitly learned to reduce the robot's velocity near obstacles, either by slowing down to a complete stop or rotating in place until the obstacles are out of sight. This emergent behavior motivates the design of the above target margin function and stop policy as the fallback policy when $l(s) > 0$. We choose $\eta = 0.01$ in our experiments.

For value-based filtering, we run the experiment across a sweep of $\epsilon$ values, and choose the $\epsilon$ with the highest return and lowest total violations. Specifically, we choose $\epsilon_\text{goal}=0.4$, and $\epsilon_\text{circle}=0.1$.

\p{State and action spaces}
The \emph{Circle} task and the \emph{Goal} task each employ ego-centric proprioceptive observations, consisting of accelerometer, velocity, angular rate, magnetic field, rear ball rotation, and local LiDAR readings. The observation space has $40$ and $72$ dimensions, respectively, while the continuous action space is similar across both tasks:

\begin{align*}
    s_{\text{\emph{Goal}}} &:= \Big[
        a_{x, y, z}, v_{x, y, z}, \;
        \omega_{x, y, z}, m_{x, y, z}\;
        \dot{b}_{x,r}, \dot{b}_{y,r}, \dot{b}_{z,r}, \;
        q_{b,r}^{(3\times3)}, \;
        \ell^{\text{goal}}_{1:16}, \;
        \ell^{\text{pillar}}_{1:16}, \;
        \ell^{\text{wall}}_{1:16}
    \Big], \\[4pt]
    s_{\text{\emph{Circle}}} &:= \Big[
        a_{x, y, z}, v_{x, y, z}, \;
        \omega_{x, y, z}, m_{x, y, z}\;
        \dot{b}_{x,r}, \dot{b}_{y,r}, \dot{b}_{z,r}, \;
        q_{b,r}^{(3\times3)}, \;
        \ell^{\text{sigwall}}_{1:16}
    \Big], \\[4pt]
    a &:= \Big[
        \tau_L, \tau_R
    \Big],
\end{align*}

where $(a_{x, y, z})$, $(v_{x, y, z})$, and $(\omega_{x, y, z})$ denote accelerometer, velocimeter, and gyroscope readings; 
$(m_{x, y, z})$ the magnetometer readings; $\dot{b}_{(\cdot),r}$ the rear ball angular velocity; 
$q_{b,r}^{(3\times3)}$ the rear ball orientation matrix; and 
$\ell^{(\cdot)}_{1:16}$ the 16-beam LiDAR signals for different object classes (goal, wall, pillar). The action vector $a=[\tau_L,\tau_R]$ applies continuous torques to the left and right wheels, each bounded within $[-1,1]$.

\p{Goal task margin function}
Let $z^{\text{pillar}}_{i},z^{\text{wall}}_{i}\in[0,1]$ denote 16-beam LiDAR
\emph{proximities} (1 = very close, 0 = far) for pillars and walls,
respectively. With LiDAR max range $R>0$ and safety clearance $\delta>0$, convert to distances
\[
d^{\text{pillar}}_{i}=(1-z^{\text{pillar}}_{i})\,R,\qquad
d^{\text{wall}}_{i}=(1-z^{\text{wall}}_{i})\,R,
\]
and define the minimum distance to safety-critical objects
\[
d_{\min}(s)\;=\;\min_i\Big\{\,d^{\text{pillar}}_{i},\,d^{\text{wall}}_{i}\,\Big\}.
\]
The margin is
\[
g_\text{goal}(s)\;=\;d_{\min}(s)\;-\;\delta.
\]
In our experiments we use $R=3.0$\,m and $\delta=0.02$\,m. The episode
terminates if $g_{\text{goal}}(s)<0$ or upon explicit collision.

\p{Circle task margin function}
Let $x(s)\in\mathbb{R}^2$ be the robot planar position and
$\mathcal{W}=\{W_k\}$ the set of signal-wall line segments (axis-aligned).
Define the point–segment distance
\[
d(x,W_k)\;=\;\min_{p\in W_k}\,\|x-p\|_2,\qquad
d_{\min}(s)\;=\;\min_{k} d\big(x(s),W_k\big).
\]
The margin is
\[
g_{\text{circle}}(s)\;=\;d_{\min}(s)\;-\;\delta.
\]

We use $R=6.0$\,m and $\delta=0.02$\,m. The episode terminates if $g_{\text{circle}}(s)<0$ or upon explicit collision.

\end{document}